\documentclass[runningheads]{llncs}
\usepackage[T1]{fontenc}
\usepackage{graphicx}
\usepackage{amssymb}
\usepackage{amsmath}
\usepackage[mathscr]{euscript}
\usepackage[small]{caption}
\usepackage{mathtools}
\usepackage{dsfont}
\usepackage{graphicx}
\usepackage{xcolor}
\usepackage{float}
\usepackage{comment}
\usepackage{tikz}
\usetikzlibrary{decorations.pathreplacing}
\usetikzlibrary{backgrounds} 
\usepackage{hyperref}
\usepackage{longtable}
\usepackage{pdflscape}
\usepackage{placeins}
\graphicspath{{figures/}}
\usepackage{multicol}

\usepackage{lineno}


\newcounter{as}[section]

\newcommand{\mcH}{\mathcal{H}}
\newcommand{\mcG}{\mathcal{G}}
\newcommand{\mcF}{\mathcal{F}}

\newcommand{\mcK}{\mathcal{K}}
\newcommand{\mcKm}{\mathcal{K}_{S}}
\newcommand{\mcP}{\mathcal{P}}
\newcommand{\mcC}{\mathcal{C}}
\newcommand{\mcCm}{\mathcal{C}_{S}}
\newcommand{\mcB}{\mathcal{B}}
\newcommand{\mcI}{\mathcal{I}}

\newcommand{\Pm}{\Psi_{M}}

\newcommand{\ME}{M^{E}}
\newcommand{\Sm}{\Sigma_{S}}
\newcommand{\SmW}{\Sigma_{S}^{W}}
\newcommand{\tphi}{\tilde{\phi}}
\newcommand{\tpsi}{\tilde{\psi}}
\newcommand{\Pw}{\Phi^{W}}
\newcommand{\msF}{\mathscr{F}}
\newcommand{\msX}{\mathscr{X}}

\newcommand{\PhiW}{\Phi^{W}}

\begin{document}
	\title{On the representation of stack operators by mathematical morphology}
	\author{Diego Marcondes\orcidID{0000-0002-6087-4821}}
	\authorrunning{D. Marcondes}
	\institute{Mathematical Sciences Institute and IRL FAMSI, The Australian National University, Canberra, Australia \email{diego.marcondes@anu.edu.au}}
	\maketitle              % typeset the header of the contribution
	\begin{abstract}
		This paper introduces the class of grey-scale image stack operators as those that (a) map binary-images into binary-images and (b) commute on average with cross-sectioning. Equivalently, stack operators are 1-Lipchitz extensions of set operators which can be represented by applying a characteristic set operator to the cross-sections of the image and adding. In particular, they are a generalisation of stack filters, for which the characteristic set operators are increasing. Our main result is that stack operators inherit lattice properties of the characteristic set operators. We focus on the case of translation-invariant and locally defined stack operators and show the main result by deducing the characteristic function, kernel, and basis representation of stack operators. The results of this paper have implications on the design of image operators, since imply that to solve some grey-scale image processing problems it is enough to design an operator for performing the desired transformation on binary images, and then considering its extension given by a stack operator. We leave many topics for future research regarding the machine learning of stack operators and the characterisation of the image processing problems that can be solved by them.
		
		\keywords{image processing  \and mathematical morphology \and stack filters.}
	\end{abstract}
	
	\section{Introduction}
	\label{Sec1}
	
	Stack filters, proposed by \cite{wendt1986stack} in the 80s, is a classical family of filters for signal and image processing. Frameworks for learning stack filters from data have been extensively studied \cite{coyle1988stack,coyle1989optimal,hirata1999design} and many algorithms to design stack filters have been proposed \cite{dellamonica2007exact,yoo1999fast}. In the context of mathematical morphology, properties of stack filters have been deduced in two important works by Maragos and Schafer \cite{maragos1987morphological,maragos2003morphological}.
	
	It is well-known that an operator $\psi$ acting on functions $f: E \mapsto \{0,\dots,m\}$ (i.e., grey-scale images), for a set $E$ and $m \in \mathbb{Z}_{+}$, is a stack filter if, and only if, there exists an \textit{increasing} set operator $\tpsi$, that acts on functions $X: E \mapsto \{0,1\}$ (i.e., binary images), such that
	\begin{align}
		\label{eq_sf}
		\psi(f) = \sum_{t = 1}^{m} \tpsi(T_{t}[f])
	\end{align}
	for all $f \in \{0,\dots,m\}^{E}$, in which $T_{t}[f](x) = \mathds{1}\{f(x) \geq t\},x \in E,$ is the cross-section of $f$ at level $t \in \{0,\dots,m\}$. We refer to \cite{maragos1987morphological,maragos2003morphological} for a proof of this fact.
	
	In this paper, we study the class of \textit{stack operators} that are the operators $\psi$ that satisfy \eqref{eq_sf} for a not necessarily increasing set operator $\tpsi$. We show that stack operators can be equivalently defined as those that map binary functions into binary functions and satisfy
	\begin{linenomath}
		\begin{equation}
			\label{stack_formula}
			\psi(f) = \frac{1}{m} \sum_{t = 1}^{m} \psi(mT_{t}[f]),
		\end{equation}
	\end{linenomath}
	that is, $\psi(f)$ is the average of $\psi$ applied to the cross-sections of $f$ rescaled to a function taking values in $\{0,m\}$. We interpret \eqref{stack_formula} as $\psi$ \textit{commuting on average} with cross-sectioning. In particular, we show that stack operators are an extension of set operators (i.e, binary image operators) to function operators (i.e., grey-scale image operators).
	
	Our main result is that the stack operator $\psi$ inherits many properties of the respective set operator $\tpsi$, such as being an erosion, a dilation, an anti-erosion, an anti-dilation, increasing, decreasing, extensive, anti-extensive, sup-generating and inf-generating. For simplicity, we focus on the case of translation-invariant and locally defined stack operators and show the main result by deducing the characteristic function, kernel, and basis representation of stack operators.
	
	The results of this paper have implications on the design of image operators, since allow to solve grey-scale image processing tasks by designing an operator for performing the desired transformation on binary images and then considering its extension given by a stack operator via \eqref{eq_sf}. In particular, the results can be leveraged to develop learning methods for grey-scale image processing based on deep neural networks and discrete morphological neural networks \cite{marcondes2024discrete}.
	
	We consider two examples of stack operators designed via \eqref{eq_sf}, one for image denoising and another for boundary recognition, both acting on functions $f: \mathbb{Z}^{2} \mapsto \{0,\dots,m\}$ and with associated set operators acting on functions $X: \mathbb{Z}^{2} \mapsto \{0,1\}$. The denoising stack operator, denoted by $\psi_{1}$, is the extension of the alternating-sequential filter $\tilde{\psi}_{1} = \epsilon_{W} \, \delta_{W} \, \delta_{W} \, \epsilon_{W}$, in which $W$ is the 3 x 3 square centred at the origin of $\mathbb{Z}^{2}$, and $\epsilon_{W}$ and $\delta_{W}$ are the set erosion and dilation with structuring element $W$, respectively. This set operator is obtained by applying the opening $\delta_{W} \, \epsilon_{W}$ followed by the closing $\epsilon_{W} \, \delta_{W}$, and is known to perform denoising. Since this set operator is increasing, the associated stack operator is actually a stack filter.  An example of the action of this set operator, and the associated stack operator, is presented in Figure \ref{fig_examples} (d). Observe that the stack filter $\psi_{1}$ inherits the denoising property of $\tilde{\psi}_{1}$.
	
	The boundary recognition stack operator, denoted by $\psi_{2}$, is obtained from a translation-invariant and locally defined set operator\footnote{See Section \ref{Sec2} for a formal definition of translation-invariant and locally defined set operators.} satisfying $\tilde{\psi}_{2}(X)(x) = 0$ if $X(y_{1}) = X(y_{2}) \ \forall y_{1},y_{2} \in x + W$ and equals $1$ otherwise, for $x \in \mathbb{Z}^{2}$, in which $x + W$ is the 3 x 3 square centred at $x$. In other words, $\tilde{\psi}_{2}(X)(x)$ equals $0$ if all pixels of $X$ in a neighbourhood of $x$ have the same value, and equals $1$ otherwise. We note that this set operator is not increasing, so the associated stack operator is not a stack filter. Furthermore, as can be seen in the example in Figure \ref{fig_examples} (c), this set operator has the effect of boundary recognition, a property that is inherited by the associated stack operator. In Figure \ref{fig_examples} (e) we present the result obtained by applying $\tilde{\psi}_{2}\tilde{\psi}_{1}$ and $\psi_{2}\psi_{1}$, that is, boundary recognition after denoising. Even though we focus on these examples, the class of stack operators encompasses other operators of interest, such as grey-scale top hat transforms, and morphological gradients, with flat structuring elements.
	
	\begin{figure}[t]
		\centering
		\includegraphics[width=0.8\linewidth]{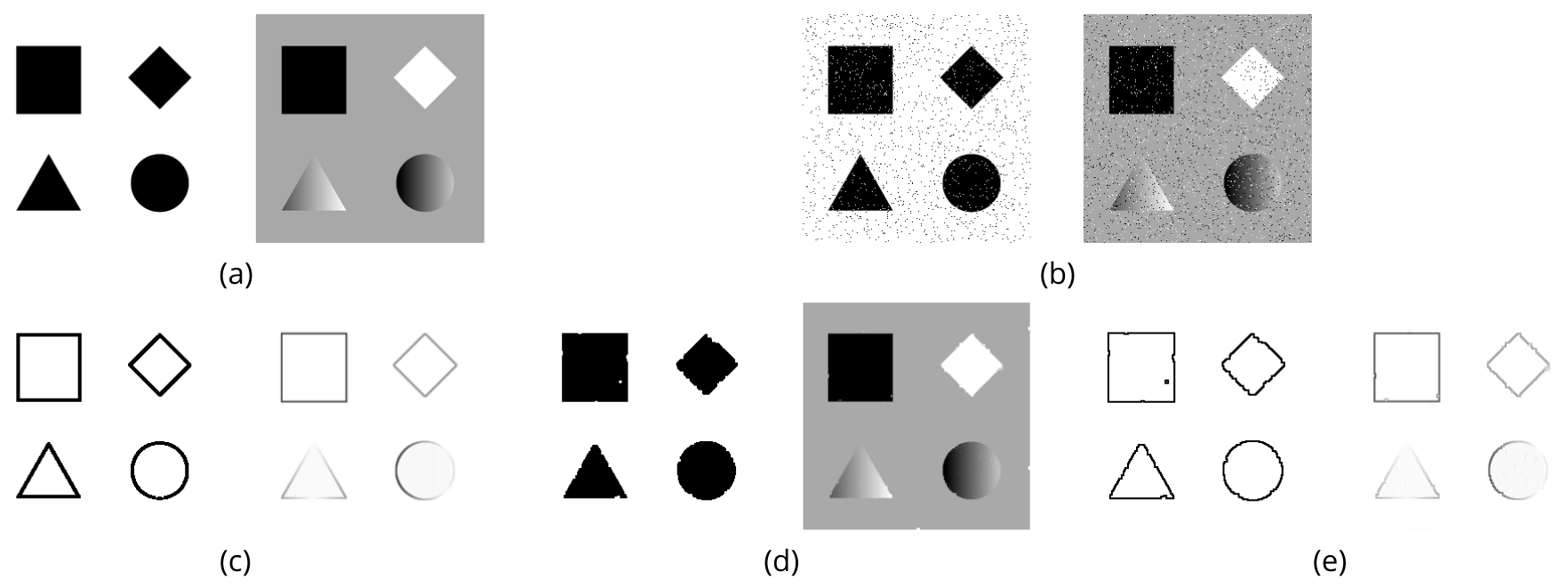}
		\caption{\footnotesize (a) Binary and grey-scale input images. (b) Binary and grey-scale input images with 2.5\% salt and pepper noise. (c) Boundary recognition of (a) by the set operator $\tilde{\psi}_{2}$ defined in Section \ref{Sec1}, and associated stack operator $\psi_{2}$. (d) Noise filtering of (b) by the alternating-sequential filter $\tilde{\psi}_{1} = \epsilon_{W} \, \delta_{W} \, \delta_{W} \, \epsilon_{W}$, and the associated stack filter $\psi_{1}$. (e) Result obtained by applying $\tilde{\psi}_{2}\tilde{\psi}_{1}$ and $\psi_{2}\psi_{1}$ to the images in (b). All images are 256 x 256 pixels and $m = 255$.} \label{fig_examples}
	\end{figure}
		
	In Section \ref{Sec2}, we define the concepts necessary to introduce the stack operators in Section \ref{Sec3}. In Section \ref{Sec6}, we establish the characteristic function, kernel, and basis representation of stack W-operators. In Section \ref{Sec7}, we state and prove the main result of this paper, and we discuss the implications of our results to the design of image operators. In Section \ref{Sec9}, we give directions for future research. Proofs for the main results are presented in the Appendix \ref{SecProofs}.
		
	\section{Grey-scale image and set operators}
	\label{Sec2}
	
	Let $E$ be a countable set, $(E,+)$ be an Abelian group, with zero element $o \in E$, and $M = \{0,\dots,m\}$ for a positive integer $m$ fixed. A function $f: E \mapsto M$ is a grey-scale image, in which $f(x)$ represents the intensity of the pixel $x \in E$. In particular, a grey-scale binary image is such that $f(x) \in \{0,m\}$ for all $x \in E$. The set $\ME$ of all grey-scale images is equipped with a partial order satisfying $f,g \in \ME, f \leq g \iff f(x) \leq g(x), \forall x \in E$. The \textit{image cross-section, or threshold,} at level $t$ is a mapping $T_{t}: \ME \mapsto \{0,1\}^{E}$ given by, for any $f \in \ME, t \in M$ and $x \in E$, $T_{t}[f](x) = \mathds{1}\{f(x) \geq t\}$. An image operator is a mapping $\psi: \ME \mapsto \ME$ and the set of all such operators, denoted by $\Pm$, is equipped with a partial order satisfying $\psi_{1},\psi_{2} \in \Pm, \psi_{1} \leq \psi_{2} \iff \psi_{1}(f) \leq \psi_{2}(f), \forall f \in \ME$. The poset $(\Psi_{M},\leq)$ is a complete lattice.
	
	A binary image can be represented as a function $X: E \mapsto \{0,1\}$ in which $X(x) = 1$ if, and only if, $x \in E$ is a pixel in the image foreground. The set $\{0,1\}^{E}$ of all binary images equipped with the partial order satisfying $X,Y \in \{0,1\}^{E}, X \leq Y \iff X(x) \leq Y(x), \forall x \in E$,	is a complete lattice. Mappings from $\{0,1\}^{E}$ into itself are called set operators, a nomenclature inspired by the fact that $\{0,1\}^{E}$ is isomorphic to the power-set of $E$. We denote the space of all set operators by $\Psi$ and equip it with the point-wise partial order $\tilde{\psi}_{1},\tilde{\psi}_{2} \in \Psi, \tilde{\psi}_{1} \leq \tilde{\psi}_{2} \iff \tilde{\psi}_{1}(X) \leq \tilde{\psi}_{2}(X), \forall X \in \{0,1\}^{E}$. An image operator is an extension of a set operator if they coincide on binary images. We note that the denoising and boundary recognition operators $\psi_{1}$ and $\psi_{2}$ introduced in Section \ref{Sec1} are extensions of the respective set operator.
	
	\begin{definition}
		\label{def_extension}
		An operator $\psi \in \Psi_{M}$ is an extension of $\tpsi \in \Psi$ if $m^{-1} \psi(mX) = \tpsi(X)$ for all $X \in \{0,1\}^{E}$. In particular, an image operator that is the extension of a set operator maps grey-scale binary images into grey-scale binary images. 
	\end{definition}
	
	For each $h \in E$, denote by $\tau_{h}: E \mapsto E$ the translation by $h$ function that maps $x \mapsto x + h$. An operator $\psi \in \Pm$, and a set operator $\tilde{\psi} \in \Psi$, are translation-invariant if they commute with translation, that is, $\psi(f \circ \tau_{h}) = \psi(f) \circ \tau_{h}$ and $\tilde{\psi}(X \circ \tau_{h}) = \tilde{\psi}(X) \circ \tau_{h}$ for all $f \in \ME$, $X \in \{0,1\}^{E}$ and $h \in E$, in  which $\circ$ means function composition.
	
	For $W \in \{0,1\}^{E}$, denote by $X|_{W}$ and $f|_{W}$ the restriction of $X \in \{0,1\}^{E}$ and $f \in \ME$ to the set $\chi(W) \coloneqq \{x \in E:W(x) = 1\}$. An operator $\psi \in \Pm$, and a set operator $\tilde{\psi} \in \Psi$, are locally defined within $W$ if\footnote{The restriction $f|_{W \circ \tau_{-x}}$ should be understood as being equal to $f(y)$ for $y \in \chi(W \circ \tau_{-x})$ and equal to $0$ for $y$ not in this set, so it is an element of $M^{E}$ and $\psi(f|_{W \circ \tau_{-x}})$ is well defined. By an analogous reason, $ \tilde{\psi}(X|_{W \circ \tau_{-x}})$ is well defined.} $\psi(f)(x) = \psi(f|_{W \circ \tau_{-x}})(x)$ and $\tilde{\psi}(X)(x) = \tilde{\psi}(X|_{W \circ \tau_{-x}})(x)$ for all $f \in \ME$, $X \in \{0,1\}^{E}$ and $x \in E$. The locally defined condition implies, for example, that the value of the transformed image $\psi(f)$ at point $x$ depends only on the value of $f$ in the neighbourhood $\chi(W \circ \tau_{-x}) = \{h \in E: W(h - x) = 1\}$ of $x$. Operators that are translation-invariant and locally defined are called W-operators. The collections of image and set W-operators are denoted by $\Pm^{W} \subseteq\Pm$ and $\Psi^{W} \subseteq\Psi$, respectively. We note that the example operators $\psi_{1}$ and $\psi_{2}$ are W-operators, the latter by construction and the former since erosions and dilations are W-operators, and composition preserves the W-operator properties, but in a greater window\footnote{See \cite[Corollary~5.1]{barrera1996set} for a proof of this fact for set operators. That it also holds for stack operators follows from Proposition \ref{prop_Wop}.}.
	
	\section{Stack operators}
	\label{Sec3}
	
	In this paper, we study the collection of stack image operators, defined as follows, and denoted by $\Sm \subseteq \Pm$.
	
	\begin{definition}
		\label{stack_operator}
		An image operator $\psi \in \Pm$ is a stack operator if, and only if, there exists a set operator $\tpsi \in \Psi$ such that \eqref{eq_sf} holds for all $f \in \ME$, or equivalently, if, and only if, (a) for all $x \in E$ and $X \in \{0,1\}^{E}$, $\psi(mX)(x) \in \{0,m\}$ and (b) relation \eqref{stack_formula} holds for all $f \in \ME$.
	\end{definition}
	
	Condition (a) of Definition \ref{stack_operator} implies that stack operators take grey-scale binary images $mX \in \{0,m\}^{E}$ into grey-scale binary images. It follows from this condition that $m^{-1}\psi(mT_{t}[f]) \in \{0,1\}^{E}$ and therefore the right-hand side of \eqref{stack_formula} is an element of $\ME$. Condition (b) then implies that $\psi(f)$ is the average of the transformation by $\psi$ of the cross-sections of $f$ represented as grey-scale binary images. Equivalently, stack operators have an associated set operator $\tilde{\psi}: \{0,1\}^{E} \mapsto \{0,1\}^{E}$ satisfying \eqref{eq_sf}, and the collection $\Psi$ of set operators is lattice isomorphic to the collection of stack image operators $\Sm$. In particular, $(\Sm,\leq)$ is a complete lattice. We state without proof this result that is a direct consequence of the definitions of $\mcF$ and $\mcF^{-1}$ below. This proposition implies that the definitions of stack operators in Definition \ref{stack_operator} are indeed equivalent.
	
	\begin{proposition}
		\label{prop_iso}
		The poset $(\Sm,\leq)$ is lattice isomorphic to $(\Psi,\leq)$ with isomorphism $\mcF: \Psi \mapsto \Sm$ given by $\mcF(\tpsi)(f) = \sum_{t = 1}^{m} \tpsi(T_{t}[f])$ and $\mcF^{-1}(\psi)(X) = m^{-1} \psi(mX)$ for $\tpsi \in \Psi$, $\psi \in \Sm$, $f \in \ME$ and $X \in \{0,1\}^{E}$.
	\end{proposition}
	
	We call $\mcF^{-1}(\psi)$ the \textit{characteristic set operator} of $\psi \in \Sm$. It follows from Proposition \ref{prop_iso} that $m^{-1} \psi(mX) = \mcF^{-1}(\psi)(X)$ and therefore $\psi$ is an extension of its characteristic set operator. Moreover, this extension is 1-Lipschitz considering the $L^{1}$-norm in the domain of $\psi$ and the $L^{\infty}$-norm in its image, that is\footnote{Observe that the Lipschitz inequality is tight since, for example, equality is attained when $\psi$ is the identity operator and $g = f + s\mathds{1}_{\{x\}}$ for $s \in \{1,\dots,m\}$ and $x \in E$ fixed.}:
	$\lVert \psi(f) - \psi(g) \rVert_{\infty} \leq \lVert f - g \rVert_{1}$ for $\psi \in \Sm, f,g \in \ME$. In particular, stack operators are $(L^{1},L^{\infty})$-continuous. A proof for the next corollary is in Appendix \ref{SecProofs}.
	
	\begin{corollary}
		\label{corollarySO}
		A stack operator $\psi \in \Sm$ is a 1-Lipschitz extension of its characteristic set operator $\mcF^{-1}(\psi)$. 
	\end{corollary}

	An important class of stack operators is that of the stack W-operators, denoted by $\Sm^{W} = \Pm^{W} \cap \Sm$. The isomorphism $\mcF$ preserves the W-operator properties, and hence maps set W-operators to stack W-operators. We state without proof this result that is a direct consequence of the definitions of $\mcF$ and $\mcF^{-1}$. Clearly, the example operators $\psi_{1}$ and $\psi_{2}$ inherit the W-operator properties of their characteristic set operators.
	
	\begin{proposition}
		\label{prop_Wop}
		An operator $\psi \in \Sm$ is a W-operator if, and only if, $\mcF^{-1}(\psi)$ is a W-operator.
	\end{proposition}
	
	The \textit{stack filters} have the property of commuting with threshold in the following sense:
	\begin{linenomath}
		\begin{align}
			\label{filter}
			T_{t}[\psi(f)] = m^{-1} \psi(mT_{t}[f])
		\end{align}
	\end{linenomath}
	for all $t \geq 1$ and $f \in \ME$. Indeed, it follows from the definition of $\mcF^{-1}$ that a stack operator $\psi$ with characteristic set operator $\tilde{\psi}$ satisfies \eqref{filter} if, and only if, $T_{t}[\psi(f)] = \tilde{\psi}(T_{t}[f])$ for all $t \geq 1$ and $f \in \ME$, a condition that is known to equivalently define stack operators (see \cite[Definition~2.1]{hirata1999design}). 
	
	The singularity of stack filters is that they satisfy \eqref{filter} for all $t \geq 1$ while general stack operators satisfy \eqref{filter} summing over $t$ (see \eqref{stack_formula}), and this fact actually implies that the stack filters are the stack operators with increasing characteristic set operator. We state this fact as a proposition whose proof is analogous to \cite[Theorem~3]{maragos1987morphological}.
	
	\begin{proposition}
		\label{prop_filter}
		A stack operator $\psi \in \Sm$ is a stack filter if, and only if, $\mcF^{-1}(\psi)$ is increasing.
	\end{proposition}
	
	\section{Representation of stack W-operators}
	\label{Sec6}
	
	The representation of set and lattice operators by mathematical morphology is a classical research topic, and results about the representation of increasing set operators \cite{maragos1987morphological}, translation-invariant set operators \cite{banon1991minimal}, and general lattice operators \cite{banon1993decomposition} have been established. In this section, we deduce the representation of stack W-operators by a characteristic function, a kernel, and a basis, and establish isomorphisms with the respective representation of the associated set W-operator, that are illustrated in Figure \ref{lattice_iso}.
	
	\begin{figure}[t]
		\centering
		\begin{tikzpicture}[scale=0.5]
			\tikzstyle{hs} = [circle,draw=black, rounded corners,minimum width=3em, vertex distance=2.5cm, line width=1pt]
			\tikzstyle{hs2} = [circle,draw=black,dashed, rounded corners,minimum width=3em, vertex distance=2.5cm, line width=1pt]
			
			\node (psiW) at (0,0) {$\Psi^{W}$};
			\node (phiW) at (5,0) {$\Pw$};
			\node (kW) at (10,0) {$\kappa^{W}$};
			\node (piW) at (15,0) {$\Pi^{W}$};
			
			\node (SmW) at (0,-5) {$\SmW$};
			\node (phiMSW) at (5,-5) {$\Phi_{S}^{W}$};
			\node (kMW) at (10,-5) {$\kappa^{W}_{S}$};
			\node (piMW) at (15,-5) {$(\Pi^{W}_{S})^{M}$};
			
			\begin{scope}[line width=1pt]
				\draw[->,line width = 0.6] (psiW) to[bend left]  node[sloped,above] {$\mcC$} (phiW);
				\draw[->,line width = 0.6] (phiW) to[bend left]  node[sloped,below] {$\mcC^{-1}$} (psiW);
				
				\draw[->,line width = 0.6] (phiW) to[bend left]  node[sloped,above] {$\mcK$} (kW);
				\draw[->,line width = 0.6] (kW) to[bend left]  node[sloped,below] {$\mcK^{-1}$} (phiW);
				
				\draw[->,line width = 0.6] (kW) to[bend left]  node[sloped,above] {$\mcB$} (piW);
				\draw[->,line width = 0.6] (piW) to[bend left]  node[sloped,below] {$\mcB^{-1}$} (kW);
				
				\draw[->,line width = 0.6] (SmW) to[bend left]  node[sloped,above] {$\mcCm$} (phiMSW);
				\draw[->,line width = 0.6] (phiMSW) to[bend left]  node[sloped,below] {$\mcCm^{-1}$} (SmW);
				
				\draw[->,line width = 0.6] (phiMSW) to[bend left]  node[sloped,above] {$\mcKm$} (kMW);
				\draw[->,line width = 0.6] (kMW) to[bend left]  node[sloped,below] {$\mcKm^{-1}$} (phiMSW);
				
				\draw[->,line width = 0.6] (kMW) to[bend left]  node[sloped,above] {$\mcB_{S}$} (piMW);
				\draw[->,line width = 0.6] (piMW) to[bend left]  node[sloped,below] {$\mcB^{-1}_{S}$} (kMW);
				
				\draw[->,line width = 0.6,dashed] (psiW) to[bend left]  node[right] {$\mcF$} (SmW);
				\draw[->,line width = 0.6,dashed] (SmW) to[bend left]  node[left] {$\mcF^{-1}$} (psiW);
				
				\draw[->,line width = 0.6,dashed] (phiW) to[bend left]  node[right] {$\mcG$} (phiMSW);
				\draw[->,line width = 0.6,dashed] (phiMSW) to[bend left]  node[left] {$\mcG^{-1}$} (phiW);
				
				\draw[->,line width = 0.6,dashed] (kW) to[bend left]  node[right] {$\mcH$} (kMW);
				\draw[->,line width = 0.6,dashed] (kMW) to[bend left]  node[left] {$\mcH^{-1}$} (kW);
				
				\draw[->,line width = 0.6,dashed] (piW) to[bend left]  node[right] {$\mcI$} (piMW);
				\draw[->,line width = 0.6,dashed] (piMW) to[bend left]  node[left] {$\mcI^{-1}$} (piW);
			\end{scope}
		\end{tikzpicture}
		\caption{\footnotesize Lattice isomorphisms between the representations of set (first row) and stack (second row) W-operators. The dashed lines refer to the isomorphisms that associate the representation of set operators with that of stack operators. } \label{lattice_iso}
	\end{figure}

	There are at least two benefits of deducing representations for stack operators. On the one hand, a representation by the combination of simpler operators may yield a more efficient algorithm to compute the operator. This benefit is evident in the representation of the example operator $\psi_{1}$ by its characteristic set operator. This representation, that is a consequence of isomorphism $\mcF$, allows computing $\psi_{1}$ by applying erosions and dilations in sequence to the image cross-sections and then adding, what might be more efficient than other representations. Furthermore, if the basis (see Section \ref{SecBasis}) of the alternating-sequential set operator $\tpsi_{1}$ was computed, then the operator could be implemented by computing $\tpsi_{1}$ as the supremum of erosions\footnote{The fact that increasing set operators can be represented as the supremum of erosions has been proved in \cite{maragos1989representation}.}, what is more computationally efficient. This representation is a consequence of the isomorphism $\mcF \circ \mcC^{-1} \circ \mcK^{-1} \circ \mcB^{-1}$ that maps the basis of the characteristic set operator to the stack operator, as will be proved in the next sections. We also note that certain representations may yield more efficient algorithms for learning operators from data; see for example the algorithm developed in \cite{marcondes2024algorithm,marcondes2025unrestricted} to learn the composition of set W-operators from data, relying on the representation via the characteristic function (see Section \ref{SecCF}).
	
	On the other hand, considering the main result of this paper (i.e. Theorem \ref{theorem}), since stack operators inherit properties of the characteristic set operators, in order to design stack operators with desired properties, it is enough to do so on the space of set operators, and it may be easier to enforce desired properties on set operators by considering suitable representations for them. For example, via the basis representation, it is easy to enforce the increasing property by representing the operator as the supremum of erosions. With this representation, it remains to design the structuring elements of these erosions, what can be done manually or by learning from data. Likewise, other restrictions may be considered on the basis of the characteristic set operator to yield desired properties, and the intervals of the basis (see Section \ref{SecBasis}) may be designed under these constraints (see Section \ref{Sec7} for more details). 	
	
	\subsection{Representation by characteristic function}
	\label{SecCF}
	
	Define $\mathcal{P}(W) \coloneqq \{0,1\}^{\chi(W)}$ and $\mathcal{P}_{M}(W) \coloneqq M^{\chi(W)}$ as the collection of binary and grey-scale patches with support in $\chi(W)$, respectively. For each $\tilde{\psi} \in \Psi^{W}$ and $\psi \in \Sm^{W}$, define the characteristic functions $\phi_{\psi}: \mathcal{P}_{M}(W) \mapsto M$ and $\tilde{\phi}_{\tilde{\psi}}: \mathcal{P}(W) \mapsto \{0,1\}$ as $\phi_{\psi}(f) = \psi(f)(o)$ and $\tilde{\phi}_{\tilde{\psi}}(X) = \tilde{\psi}(X)(o)$ for $f \in \mathcal{P}_{M}(W)$ and $X \in \mathcal{P}(W)$, recalling that $o$ is the zero element of $(E,+)$. The characteristic function determines the value of each pixel $x \in E$ in the transformed image from the values of the pixels of the original image in its neighbourhood $x + W$. 
	
	Define $\Phi_{S}^{W} \coloneqq \{\phi_{\psi}: \psi \in \SmW\} \subsetneq M^{\mathcal{P}_{M}(W)}$ and $\Phi^{W} \coloneqq \{0,1\}^{\mathcal{P}(W)}$ as the space of characteristic functions of stack and set W-operators, respectively, and consider in them the point-wise partial orders:
	\begin{linenomath}
		\begin{align*}
			&\phi_{1},\phi_{2} \in \Phi_{S}^{W}, \phi_{1} \leq \phi_{2} \iff \phi_{1}(f) \leq \phi_{2}(f), \ \forall f \in \mathcal{P}_{M}(W),\\
			&\tilde{\phi}_{1},\tilde{\phi}_{2} \in \PhiW, \tilde{\phi}_{1} \leq \tilde{\phi}_{2} \iff \tilde{\phi}_{1}(X) \leq \tilde{\phi}_{2}(X), \ \forall X \in \mathcal{P}(W).
		\end{align*}
	\end{linenomath}
	The transformations $\mcC: \Psi^{W} \mapsto \Phi^{W}$ and $\mcC_{S}: \SmW \mapsto \Phi_{S}^{W}$, that map stack and set W-operators into their characteristic functions as
	\begin{linenomath}
		\begin{align*}
			&\mcC(\tpsi)(X) = \tpsi(X)(o), \ X \in \mathcal{P}(W), \tpsi \in \PhiW,\\
			&\mcC_{S}(\psi)(f) = \psi(f)(o), \ f \in \mathcal{P}_{M}(W), \psi \in \SmW,
		\end{align*}
	\end{linenomath} 
	respectively, are lattice isomorphisms. The inverse of these mappings are $$\mcC^{-1}(\tphi)(X)(x) = \tphi(X|_{W \circ \tau_{-x}} \circ \tau_{x}) \text{ and } \mcC_{S}^{-1}(\phi)(f)(x) = \phi(f|_{W \circ \tau_{-x}} \circ \tau_{x}).$$ We state this result as a proposition whose proof is a direct consequence of the definition of the isomorphisms and the properties of W-operators.
	
	\begin{proposition}
		\label{prop_Car}
		The mappings $\mcC$ and $\mcC_{S}$ are lattice isomorphisms between $(\Psi^{W},\leq)$ and $(\Phi^{W},\leq)$, and $(\SmW,\leq)$ and $(\Phi_{S}^{W},\leq)$, respectively. 
	\end{proposition}
	
	It follows from Propositions \ref{prop_iso} and \ref{prop_Car} that there exists a lattice isomorphism $\mcG: \Phi^{W} \mapsto \Phi_{S}^{W}$ given, for example, by $\mcG = \mcC_{S} \circ \mcF \circ \mcC^{-1}$. Next lemma gives explicit formulas for $\mcG$ and $\mcG^{-1}$ that can be obtained by explicitly computing $\mcC_{S} \circ \mcF \circ \mcC^{-1}$ and its inverse.
	
	\begin{lemma}
		\label{lemmaG}
		For $\tphi \in \Phi^{W}$ and $f \in \mathcal{P}_{M}(W)$, it holds $\mcG(\tphi)(f) = \sum_{t = 1}^{m} \tphi(T_{t}[f])$ and, for $\phi \in \Phi^{W}_{S}$ and $X \in \mathcal{P}(W)$, it holds $\mcG^{-1}(\phi)(X) = m^{-1} \phi(mX)$.
	\end{lemma}
	
	It follows from Lemma \ref{lemmaG} that, as stack operators can be expressed by applying a set operator to the cross-sections of an image and adding, their characteristic function can be represented by applying the characteristic function of the associated characteristic set operator to the cross-section of ``images'' in $\chi(W)$ and adding. Furthermore, lattice isomorphism $\mcG$ implies an equivalency between the collection of stack W-operators and the functions from $\mathcal{P}(W)$ to $\{0,1\}$. In particular, it establishes the known (see \cite{hirata1999design}) equivalency between the translation-invariant and locally defined stack filters and the collection of positive (i.e., increasing) functions from $\mathcal{P}(W)$ to $\{0,1\}$.
	
	\subsection{Kernel representation}
	
	Define as $\kappa^{W} \coloneqq \mathcal{P}(\mathcal{P}(W))$ the collection of subsets of $\mathcal{P}(W)$ and $\kappa_{M}^{W} \coloneqq \mathcal{P}(\mathcal{P}_{M}(W))^{M}$ as the collection of set-valued functions from $M$ to the collection $\mathcal{P}(\mathcal{P}_{M}(W))$ of subsets of $\mathcal{P}_{M}(W)$. These sets are equipped with the partial orders $\mathscr{X}_{1},\mathscr{X}_{2} \in \kappa^{W}, \mathscr{X}_{1} \leq \mathscr{X}_{2} \iff \mathscr{X}_{1} \subseteq \mathscr{X}_{2}$ and $\mathscr{F}_{1},\mathscr{F}_{2} \in \kappa^{W}_{S}, \mathscr{F}_{1} \leq \mathscr{F}_{2} \iff \mathscr{F}_{1}(t) \subseteq \mathscr{F}_{2}(t), \ \forall t \in M$. There exist lattice isomorphisms between $\Phi^{W}$ and $\kappa^{W}$, and $\Phi_{S}^{W}$ and a subset of $\kappa_{S}^{W}$, that map each characteristic function to its kernel representation. Formally, these are mappings $\mcK: \Phi^{W} \mapsto \kappa^{W}$ and $\mcK_{S}: \Phi_{S}^{W} \mapsto \kappa_{S}^{W}$ satisfying
	\begin{linenomath}
		\begin{align*}
			\mcK(\tphi) = \{X \in \mathcal{P}(W): \tphi(X) = 1\} & & \text{ and } & & \mcK_{S}(\phi)(t) = \{f \in \mathcal{P}_{M}(W): t \leq \phi(f)\}
		\end{align*}
	\end{linenomath}
	for $\tphi \in \Phi^{W}, \phi \in \Phi_{S}^{W}$ and $t \in M$, in which $\kappa_{S}^{W} \coloneqq \{\mcK_{S}(\phi): \phi \in \Phi_{S}^{W}\} \subsetneq \kappa_{M}^{W}$.  This definition of kernel is due to \cite{banon1991minimal,banon1993decomposition}.
	
	The next proposition, that is direct from the definition of $\mcK$ and $\mcK_{S}$, shows that these are indeed lattice isomorphisms. The inverse of these mappings are $\mcK^{-1}(\msX)(X) = \mathds{1}\{X \in \msX\}$ and $\mcK_{S}^{-1}(\msF)(f) = \max\{t \in M: f \in \msF(t)\}$.
	
	\begin{proposition}
		\label{prop_kernel}
		The mappings $\mcK$ and $\mcK_{S}$ are lattice isomorphisms between, respectively, $(\Phi^{W},\leq)$ and $(\kappa^{W},\leq)$, and $(\Phi_{S}^{W},\leq)$ and $(\kappa_{S}^{W},\leq)$.
	\end{proposition}
	
	It follows from Lemma \ref{lemmaG} and Proposition \ref{prop_kernel} that there exists a lattice isomorphism $\mcH: \kappa^{W} \mapsto \kappa_{S}^{W}$, for which we give an explicit expression in the next lemma that is proved in the Appendix \ref{SecProofs}.
	
	\begin{lemma}
		\label{lemmaH}
		The mapping $\mcH: \kappa^{W} \mapsto \kappa_{S}^{W}$ given by
		\begin{linenomath}
			\begin{align*}
				\mcH(\mathscr{X})(t) = \left\{\sum_{s = 1}^{m} X_{s}: X_{s} \in \mathcal{P}(W), X_{m} \leq \cdots \leq X_{1}, \sum_{s = 1}^{m} \mathds{1}\{X_{s} \in \mathscr{X}\} \geq t\right\}
			\end{align*}
		\end{linenomath}
		for $\mathscr{X} \in \kappa^{W}$ and $t \in M$ is a lattice isomorphism between $(\kappa^{W},\leq)$ and $(\kappa_{S}^{W},\leq)$ with inverse mapping $\mcH^{-1}(\mathscr{F}) = \left\{X \in \mathcal{P}(W): mX \in \mathscr{F}(m)\right\}$ for $\mathscr{F} \in \kappa_{S}^{W}$.
	\end{lemma}
	
	\subsection{Basis representation}
	\label{SecBasis}
	
	For $f_{1},f_{2} \in \mathcal{P}_{M}(W)$ and $X_{1},X_{2} \in \mathcal{P}(W)$, the intervals with limits $f_{1},f_{2}$ and $X_{1},X_{2}$ are the sets $[f_{1},f_{2}] = \{f \in \mathcal{P}_{M}(W):f_{1} \leq f \leq f_{2}\}$ and $[X_{1},X_{2}] = \{X \in \mathcal{P}(W):X_{1} \leq X \leq X_{2}\}$, respectively, which are empty if $f_{1} \nleq f_{2}$ or $X_{1} \nleq X_{2}$. For collections $\mathscr{X} \in \mathcal{P}(\mathcal{P}(W))$ and $\mathscr{G} \in \mathcal{P}(\mathcal{P}_{M}(W))$, define the respective maximal collection as formed by the elements in it that are not lesser than any other one: $\textbf{\text{Max}}(\mathscr{X}) = \{X \in \mathscr{X}: Y \in \mathscr{X}, X \leq Y \implies X = Y\}$ and $\textbf{\text{Max}}(\mathscr{G}) = \{f \in \mathscr{G}: g \in \mathscr{G}, f \leq g \implies f = g\}$.
	
	Since set and stack W-operators can be uniquely represented by their kernel, they can be represented by any collection of sets that covers their kernel\footnote{In the case of grey-scale image operators, there should actually be a collection of sets for each $t$, that covers the kernel applied at a point $t$.}. A way of covering a subset of a lattice, is by a collection of intervals, and such a cover with the minimum number of intervals is that of maximal intervals. The basis of an operator was proposed by \cite{banon1991minimal,banon1993decomposition} based on the idea of covering the kernel (applied at point $t$ in the grey-scale case) with maximal intervals.	
	
	Formally, define by $\Pi^{W} \subseteq\mathcal{P}(\mathcal{P}(W))$ the family of collections of maximal intervals in $\mathcal{P}(W)$ and by $\Pi_{S}^{W} \subseteq\mathcal{P}(\mathcal{P}_{M}(W))$ the family of collections of maximal intervals in $\mathcal{P}_{M}(W)$. Consider in these collections the partial orders
	\begin{linenomath}
		\begin{align*}
			&I_{1},I_{2} \in \Pi^{W}, I_{1} \leq I_{2} \iff \forall [X_{1},X_{2}] \in I_{1}, \exists [Y_{1},Y_{2}] \in I_{2} \text{ s.t. } [X_{1},X_{2}] \subseteq[Y_{1},Y_{2}] \\		
			&\mathscr{I}_{1},\mathscr{I}_{2} \in \Pi^{W}_{S}, \mathscr{I}_{1} \leq \mathscr{I}_{2} \iff \forall [f_{1},f_{2}] \in \mathscr{I}_{1}, \exists [g_{1},g_{2}] \in \mathscr{I}_{2} \text{ s.t. } [f_{1},f_{2}] \subseteq[g_{1},g_{2}].
		\end{align*}
	\end{linenomath}
	The mappings $\mcB: \kappa^{W} \mapsto \Pi^{W}$ and $\mcB_{S}: \kappa^{W}_{S} \mapsto (\Pi^{W}_{S})^{M}$ defined as
	\begin{linenomath}
		\begin{align*}
			\mcB(\mathscr{X}) = \textbf{\text{Max}}(\{[X,Y]: [X,Y] \subseteq\mathscr{X}\}) & & \mcB_{S}(\mathscr{F})(t) = \textbf{\text{Max}}(\{[f,g]: [f,g] \subseteq\mathscr{F}(t)\})
		\end{align*}
	\end{linenomath}
	for $\mathscr{X} \in \kappa^{W}, \mathscr{F} \in \kappa^{W}_{S}$ and $t \in M$ are lattice isomorphisms, in which the partial order in $(\Pi^{W}_{S})^{M}$ is $\boldsymbol{\mathscr{I}}_{1},\boldsymbol{\mathscr{I}}_{2} \in (\Pi^{W}_{S})^{M}, \boldsymbol{\mathscr{I}}_{1} \leq \boldsymbol{\mathscr{I}}_{2} \iff \boldsymbol{\mathscr{I}}_{1}(t) \leq \boldsymbol{\mathscr{I}}_{2}(t), \forall t \in M$. The inverse of these transformations are $\mcB^{-1}(I) = \bigcup \{[X,Y]: [X,Y] \in I\}$ and $\mcB_{S}^{-1}(\boldsymbol{\mathscr{I}})(t) = \bigcup \{[f,g]: [f,g] \in \boldsymbol{\mathscr{I}}(t)\}$ for $I \in \Pi^{W}, \boldsymbol{\mathscr{I}} \in (\Pi^{W}_{S})^{M}$ and $t \in M$. This result is stated as a proposition that follows directly from the definitions of $\mcB$ and $\mcB_{S}$. In particular, $\mcB^{-1}$ and $\mcB_{S}^{-1}$ recover the kernel from its covering by the basis.
	
	\begin{proposition}
		\label{prop_basis}
		The mappings $\mcB$ and $\mcB_{S}$ are lattice isomorphisms.
	\end{proposition}
	
	We state that there exists a lattice isomorphism $\mcI: \Pi^{W} \mapsto (\Pi^{W}_{S})^{M}$ between the basis representation of set and stack W-operators.
	
	\begin{lemma}
		\label{lemma_isoBasis}
		The mapping $\mcI \coloneqq \mcB_{S} \circ \mcH \circ \mcB^{-1}$ is a lattice isomorphism between $\Pi^{W}$ and $(\Pi^{W}_{S})^{M}$.
	\end{lemma}
	
	In particular, the isomorphism $\mcI$ takes basis of set operators that have one interval to basis of stack operators that have one interval for every $t$. This simple result is key to proving Theorem \ref{theorem}. This result is proved in Appendix \ref{SecProofs}.
	
	\begin{corollary}
		\label{corollary_1interval}
		If $I = [X,Y]$, then $\mcI(I)(t) = \{[tX,tY + (m-t)W]\}$ for $t = 1,\dots,m$.
	\end{corollary}
	
	\section{Restricted classes of stack operators and learning from data}
	\label{Sec7}
	
	It follows from Proposition \ref{prop_iso} that restricted classes of stack operators can be obtained by constraining $\Psi$ and applying $\mcF$. Actually, $\mcF$ preserves many lattice properties of $\tilde{\psi}$, hence restricting $\Psi$ based on these properties will generate analogous restrictions in $\Sm$. The next theorem, for which a proof is presented in Appendix \ref{SecProofs}, establishes these properties for the case of stack W-operators.
	
	\begin{theorem}
		\label{theorem}
		A stack operator $\psi \in \SmW$ is (1) increasing, (2) decreasing, (3) extensive, (4) anti-extensive, (5) an erosion, (6) a dilation, (7) an anti-dilation, (8) an anti-erosion, (9) sup-generating, and (10) inf-generating	if, and only if, so is $\mcF^{-1}(\psi)$.
	\end{theorem}
	
	The isomorphisms depicted in Figure \ref{lattice_iso} together with Theorem \ref{theorem} imply that a stack W-operator can be designed by designing a set operator that performs the desired transformation on sets, what actually reduces to designing a binary function in $\mcP(W)$. This is what has been done to design the stack operators in Figure \ref{fig_examples}.
	
	Prior information about properties that the optimal stack operator satisfies can be used to define a restricted class of set operators that, due to Theorem \ref{theorem}, will generate a restricted class of stack operators with the same properties, from which an operator can be learned from data. For example, applying the isomorphism $\mcF \circ \mcC^{-1} \circ \mcK^{-1} \circ \mcB^{-1}$ to the basis $I = \{[A_{i},B_{i}]: i = 1,\dots,n\} \in \Pi^{W}$ of a set operator we get the stack operator given by
	\begin{align}
		\label{so_basis}
		\psi_{I}(f) = \sum_{t = 1}^{m} \bigvee_{i = 1}^{n} \lambda_{[A_{i},B_{i}]}(T_{t}[f]), & & f \in M^{E}
	\end{align}
	in which $\lambda_{[A_{i},B_{i}]}$ is the set operator with characteristic function $\tphi_{[A_{i},B_{i}]}(X) = 1$ if $A_{i} \leq X \leq B_{i}$ and equal to zero otherwise for $X \in \mathcal{P}(W)$, that is the operator with basis $\{[A_{i},B_{i}]\}$. This decomposition is based on the sup-generating representation of set operators developed by \cite{banon1991minimal} and $\lambda_{[A_{i},B_{i}]}$ is called a sup-generating operator.
	
	If pairs of input and target images $(f_{j},g_{j}), j = 1,\dots,N$, are available, then, based on \eqref{so_basis}, a stack operator may be learned from data by minimising the loss $L_{N}(I) = \frac{1}{N} \sum_{j=1}^{N} \lVert \psi_{I}(f_{j}) - g_{j} \rVert_{1}$ over $I$ in a subset of $\Pi^{W}$ that depends on prior knowledge about the target operator, in which the size $n$ of the basis in \eqref{so_basis} controls the \textit{complexity} of the operator. For example, to learn a stack filter, we can consider only basis $I$ with $B_{i} = W$ for all\footnote{It follows from the representation deduced in \cite{maragos1989representation} that increasing operators have basis with intervals of form $[A_{i},W]$.} $i = 1,\dots,n$. Also, to learn extensive stack operators we can consider only basis which contain the interval $[A,W]$ with $A(o) = 1$ and $A(x) = 0$ for $x \neq o$, what guarantees that $f(x) \leq \psi_{I}(f)(x)$ for all $x \in E$. Other properties may be enforced on $\psi_{I}$ by considering an inf-generating representation, under which dilations, decreasing and anti-extensive operators have simplified representations (see \cite{banon1991minimal} for more details).
	
	The minimisation of $L_{N}(I)$ can be carried out by the lattice descent algorithm proposed for training discrete morphological neural networks (DMNN) in \cite{marcondes2024discrete} yielding a method that is an extension of DMNN for grey-scale image operators. By dropping the assumption that the intervals in $I$ are maximal, the representation \eqref{so_basis} is actually an overparametrisation of the operator, since there are more than one collection of intervals representing the same operator, that are all collections that cover its kernel, but the optimisation problem becomes more efficient (see \cite{marcondes2024lattice} for more details). If a characteristic function representation for the set operator was considered instead, then, relying on the isomorphism $\mcF \circ \mcC^{-1}$, the algorithm proposed in \cite{marcondes2024algorithm,marcondes2025unrestricted} could be adapted to learn stack operators without any restriction. Also, a deep neural network might also be considered to represent the characteristic function of the set W-operator and be trained by minimising the error associated with the respective stack operator. We will investigate these learning methods in the future.
	
	\section{Future research}
	\label{Sec9}
	
	As discussed above, the next step in this research is to develop and implement learning methods for stack operators, but there are also interesting theoretical lines of research. Given the simplicity of the representation of stack W-operators, that are equivalent to Boolean functions in $\mcP(W)$, they could be the preferred choice when solving a practical image transformation problem. However, it is necessary to characterise what problems they can solve, that is, what image operators they can approximate. Therefore, the main theoretical topic for future research is to characterise the stack operators within the family of 1-Lipschitz extensions of set operators by investigating if they have universal approximating properties, i.e., by characterising spaces in which the stack operators are dense. Stack operators will be suitable to solve a problem when the optimal operator is in such a space.
	
	Also, it would be interesting to study the representation of stack operators with other kinds of invariance, such as group equivariant operators \cite{penaud2024group} and scale invariant operators \cite{sangalli2021scale}. As a generalisation, other functions that map $M$ into $\{0,1\}$ besides the threshold $T_{t}$ could be considered in representation \eqref{eq_sf} yielding more general classes of grey-scale image operators. In general cases, we could obtain an operator that is an extension of a set operator, but with other Lipschitz constant and maybe with better universal approximation properties. 	
	
	\section*{Acknowledgements}
	
	This work was partially funded by grant \#22/06211-2 S\~ao Paulo Research Foundation (FAPESP). I thank the anonymous reviewers for their comments that greatly improved the manuscript, in particular the proof of Corollary \ref{corollarySO}.
	
	\appendix
	\section{Proof of results}
	\label{SecProofs}

	\begin{proof}[Proof of Corollary \ref{corollarySO}]
		Fix $f,g \in \ME$ and $\psi \in \Sm$, and denote $\tpsi \coloneqq \mcF^{-1}(\psi)$. Define $D(f,g) = \{t \in \{1,\dots,m\}: T_{t}[f] \neq T_{t}[g]\}$ and observe that $\lVert \psi(f) - \psi(g) \rVert_{\infty} \leq \sum_{t \in D(f,g)} \lVert  (\tpsi(T_{t}[f]) - \tpsi(T_{t}[g])) \rVert_{\infty} \leq |D(f,g)|$. We will show that $|D(f,g)| \leq \lVert f - g \rVert_{1}$. If $x \in E$ is such that $|f(x) - g(x)| = s$, then $\{t_{x} + 1,\dots,t_{x} + s\} \subset D(f,g)$ in which $t_{x} = \min(f(x),g(x))$ so, denoting $\Delta_{s} = \{x \in E: |f(x) - g(x)| = s\}$, we conclude that $|D(f,g)| \leq \sum_{s=1}^{m} s |\Delta_{s}|$. The result follows since $\lVert f - g \rVert_{1} = \sum_{s = 1}^{m} s |\Delta_{s}|$.\hfill$\blacksquare$
	\end{proof}

	\begin{proof}[Proof of Lemma \ref{lemmaH}]
		We will show that $\mcH = \mcK_{S} \circ \mcG \circ \mcK^{-1}$ and $\mcH^{-1} = \mcK \circ \mcG^{-1} \circ \mcK_{S}^{-1}$ so the result follows from Lemma \ref{lemmaG} and Proposition \ref{prop_kernel}. On the one hand, for $\mathscr{F} \in \kappa_{S}^{W}$ and $X \in \mathcal{P}(W)$, it holds $(\mcG^{-1} \circ \mcK_{S}^{-1})(\mathscr{F})(X) = m^{-1} \max\{t \in M: mX \in \mathscr{F}(t)\}$, therefore $(\mcK \circ \mcG^{-1} \circ \mcK_{S}^{-1})(\mathscr{F}) = \{X \in \mathcal{P}(W): \max\{t \in M: mX \in \mathscr{F}(t)\} = m\} = \{X \in \mathcal{P}(W): mX \in \mathscr{F}(m)\}$,
		in which the last equality follows from the fact that $\mathscr{F}(t)$ is decreasing on $t$. On the other hand, for $\mathscr{X} \in \kappa^{W}, f \in \mathcal{P}_{M}(W)$ and $t \in M$, it holds $(\mcG \circ \mcK^{-1})(\mathscr{X})(f) = \sum_{t = 1}^{m} \mathds{1}\{T_{t}[f] \in \mathscr{X}\}$, therefore
		\begin{linenomath}
			\begin{align*}
				(\mcK_{S} &\circ \mcG \circ \mcK^{-1})(\mathscr{X})(t) = \left\{f \in \mathcal{P}_{M}(W): t \leq \sum_{s = 1}^{m} \mathds{1}\{T_{s}[f] \in \mathscr{X}\}\right\}\\
				&= \left\{\sum_{s = 1}^{m} X_{s}: X_{s} \in \mathcal{P}(W), X_{m} \leq \cdots \leq X_{1}, \sum_{s = 1}^{m} \mathds{1}\{X_{s} \in \mathscr{X}\} \geq t\right\}.
			\end{align*}
		\end{linenomath}
		To see that the last equality holds, take $X_{s} = T_{s}[f]$ for $f \in \mathcal{P}_{M}(W)$ and $s \in M$. \hfill$\blacksquare$
	\end{proof}

	\begin{proof}[Proof of Corollary \ref{corollary_1interval}]
		Fix $I = \{[X,Y]\}$ and observe that, for $t \geq 1$, $(\mcH \circ \mcB^{-1})(t) = \left\{\sum_{s = 1}^{t} X_{s}: X_{m} \leq \cdots X_{1},\sum_{s = 1}^{t} \mathds{1}\{X_{s} \in [X,Y]\} \geq t\right\}$. Then, $f \in (\mcH \circ \mcB^{-1})(t)$ if, and only if, $f$ can be decomposed as $f = \sum_{j=1}^{t} X_{s_{j}} + g$ with $X_{s_{j}} \in [X,Y]$ and $g(x) \leq (m-t)W$, assuming that $X_{s_{t}} \leq \cdots \leq X_{s_{1}}$. But $f$ can be decomposed as such if, and only if, $f \in [tX,tY + (m-t)W]$. We conclude that $\mcI([X,Y])(t) = \mcB_{S}(\mcH \circ \mcB^{-1})(t) = \{[tX,tY + (m-t)W]\}$.
	\end{proof}
	
	\begin{proof}[Proof of Theorem \ref{theorem}]
		For each $\psi \in \SmW$ denote $\tpsi \coloneqq \mcF^{-1}(\psi)$ and $\tphi \coloneqq \mcC(\tpsi$), and for each $\tpsi \in \Psi^{W}$ denote $\psi \coloneqq \mcF(\tpsi)$ and $\phi \coloneqq \mcCm(\psi)$.
		
		(1) \textit{Increasing}: Direct from Proposition \ref{prop_filter}. (2) \textit{Decreasing}: Analogous to increasing.
		
		(3) \textit{Extensive}: Fix $\tpsi \in \Psi^{W}$ extensive and $f \in \mathcal{P}_{M}(W)$. Then, $\psi(f) = \sum_{t = 1}^{m} \tilde{\psi}(T_{t}[f]) \geq \sum_{t = 1}^{m} T_{t}[f] = f$ in which the inequality follows since $\tilde{\psi}$ is extensive. Fix $\psi \in \SmW$ extensive and $X \in \mathcal{P}(W)$. Then, $\tilde{\psi}(X) = \frac{1}{m} \psi(mX) \geq \frac{1}{m} mX = X$.	(4) \textit{Anti-extensive}: Analogous to extensive.
		
		(5) \textit{Erosion}: Due to the main result in \cite{banon1993decomposition}, set and image operators are erosions if, and only if, their kernel is formed by only one interval of form $[X,W]$ and $[f_{t},mW]$ for all $t \geq 1$, respectively. On the one hand, it follows from Corollary \ref{corollary_1interval} that if $\tpsi$ is an erosion with kernel $[X,W]$ then the kernel of $\psi$ equals $[tX,W]$ for all $t \geq 1$ so it is an erosion. On the other hand, if $\psi$ is an erosion then $\mcK_{S}(\psi)(m) = [f_{m},mW]$ and $\mcK(\tpsi) = \{X \in \mcP(W): f_{m} \leq mX\}$. It is clear that $f_{m} \leq mX$ if, and only if, $T_{1}[f_{m}] \leq X$ and therefore $\mcK(\tpsi) = [T_{1}[f_{m}],W]$, and $\tpsi$ is an erosion.
		
		(6) \textit{Dilation}: Analogous to erosion by considering the dual partial orders, since dilations according to a partial order are erosions according to the dual partial order.
		
		(7) \textit{Anti-dilation}: Anti-dilations are operators such that $\psi(f \vee g) = \psi(f) \wedge \psi(g)$ and $\tpsi(X \vee Y) = \tpsi(X) \wedge \tpsi(Y)$. Due to the main result in \cite{banon1993decomposition}, the kernel of an anti-dilation is formed by a unique interval of form $[\emptyset,Y]$, for set W-operators, and $[\emptyset,g_{t}]$ for all $t \geq 1$, for image operators. A deduction analogous to that of erosions follows by Corollary \ref{corollary_1interval}.
		
		(8) \textit{Anti-erosion}: Anti-erosions are operators such that $\psi(f \wedge g) = \psi(f) \vee \psi(g)$ and $\tpsi(X \wedge Y) = \tpsi(X) \vee \tpsi(Y)$. This case is analogous to anti-dilations by considering the dual partial order. 
		
		(9) \textit{Sup-generating}: A sup-generating operator is the meet between an erosion and an anti-dilation (ref. \cite{banon1993decomposition}). Direct from Corollary \ref{corollary_1interval} since set and image operators are sup-generating if, and only if, their kernel is formed by only one interval.
		
		(10) \textit{Inf-generating}: An inf-generating operator is the join between a dilation and an anti-erosion (ref. \cite{banon1993decomposition}). Analogous to sup-generating by considering the dual partial orders, since sup-generating operators according to a partial order are inf-generating according to the dual partial order.
		
		\hfill$\blacksquare$
	\end{proof}
	
	\vspace{-0.75cm}
	
	\bibliographystyle{splncs04}
	\bibliography{Ref}
\end{document}